\documentclass{article}

\usepackage{microtype}
\usepackage{graphicx}
\usepackage{caption}
\usepackage{subcaption}
\usepackage{booktabs} % for professional tables

\usepackage{hyperref}

\usepackage[accepted]{icml2025}

\usepackage{amsmath}
\usepackage{amssymb}
\usepackage{mathtools}
\usepackage{amsthm}

\long\def\comment#1{}

\newfont{\bbb}{msbm10 scaled 700}

\newfont{\bb}{msbm10 scaled 1100}

\newcommand{\EE}{\mbox{\bb E}}

\newcommand{\Ac}{{\cal A}}

\newcommand{\Dc}{{\cal D}}

\newcommand{\Nc}{{\cal N}}

\newcommand{\Sc}{{\cal S}}

\DeclareMathOperator*{\argmax}{arg\,max}
\DeclareMathOperator*{\argmin}{arg\,min}
\usepackage{mathtools}

\usepackage[capitalize,noabbrev]{cleveref}

\theoremstyle{plain}
\newtheorem{theorem}{Theorem}[section]

\theoremstyle{definition}

\theoremstyle{remark}

\usepackage[textsize=tiny]{todonotes}

\icmltitlerunning{Moderate Actor-Critic Methods}

\begin{document}

\twocolumn[
%\icmltitle{Controlling Overestimation Bias in Actor-Critic Methods\\ Via Expectile Loss}
\icmltitle{Moderate Actor-Critic Methods:\\ Controlling Overestimation Bias via Expectile Loss}



\icmlsetsymbol{equal}{*}

\begin{icmlauthorlist}
\icmlauthor{Ukjo Hwang}{hanyang}
\icmlauthor{Songnam Hong}{hanyang}
%\icmlauthor{Firstname3 Lastname3}{comp}
%\icmlauthor{Firstname4 Lastname4}{sch}
%\icmlauthor{Firstname5 Lastname5}{yyy}
%\icmlauthor{Firstname6 Lastname6}{sch,yyy,comp}
%\icmlauthor{Firstname7 Lastname7}{comp}
%\icmlauthor{Firstname8 Lastname8}{sch}
%\icmlauthor{Firstname8 Lastname8}{yyy,comp}
\end{icmlauthorlist}

\icmlaffiliation{hanyang}{Department of Electronic Engineering, Hanyang University, Seoul, Korea}
%\icmlaffiliation{comp}{Company Name, Location, Country}
%\icmlaffiliation{sch}{School of ZZZ, Institute of WWW, Location, Country}

\icmlcorrespondingauthor{Songnam Hong}{snhong@hanyang.ac.kr}
%\icmlcorrespondingauthor{Firstname2 Lastname2}{first2.last2@www.uk}

\icmlkeywords{Machine Learning, ICML}

\vskip 0.3in
]




%This document provides a basic paper template and submission guidelines.
%Abstracts must be a single paragraph, ideally between 4--6 sentences long.
%Gross violations will trigger corrections at the camera-ready phase.
\begin{abstract}
Overestimation is a fundamental characteristic of model-free reinforcement learning (MF-RL), arising from the principles of temporal difference learning and the approximation of the Q-function. To address this challenge, we propose a novel moderate target in the Q-function update, formulated as a convex optimization of an overestimated Q-function and its lower bound. Our primary contribution lies in the efficient estimation of this lower bound through the lower expectile of the Q-value distribution conditioned on a state. Notably, our moderate target integrates seamlessly into state-of-the-art (SOTA) MF-RL algorithms, including Deep Deterministic Policy Gradient (DDPG) and Soft Actor Critic (SAC). Experimental results validate the effectiveness of our moderate target in mitigating overestimation bias in DDPG, SAC, and distributional RL algorithms.
%Overestimation is an inherent characteristic of model-free reinforcement learning (MF-RL), arising from the fundamental principles of temporal difference learning and the approximation of Q-function. To alleviate this issue, we propose a novel {\em moderate} target in the Q-function update, which is constructed as a convex combination of an overestimated Q-function and its lower bound. Our major contribution is the efficient estimation of this lower bound via the lower expectile of the Q-value distribution conditioned on a state.
%Our key contribution is to efficiently estimate this lower bound through the lower expectile of the Q-value distribution. Remarkably, our moderate target can be seamlessly integrated into the state-of-the-art (SOTA) MF-RL algorithms, such as Deep Deterministic Policy Gradient (DDPG), Soft Actor Critic (SAC). Through experiments conducted on MuJoCo continuous control tasks, we validate the effectiveness of our moderate target in alleviating overestimation bias in DDPG, SAC, and SOTA distributional RL.
\end{abstract}

%%%%%%%%%%%%%%%%%%%%%%%%%%%%%%%%%%%%
\begin{figure*}[t]
   \centering  
   \includegraphics[width=0.80\linewidth]{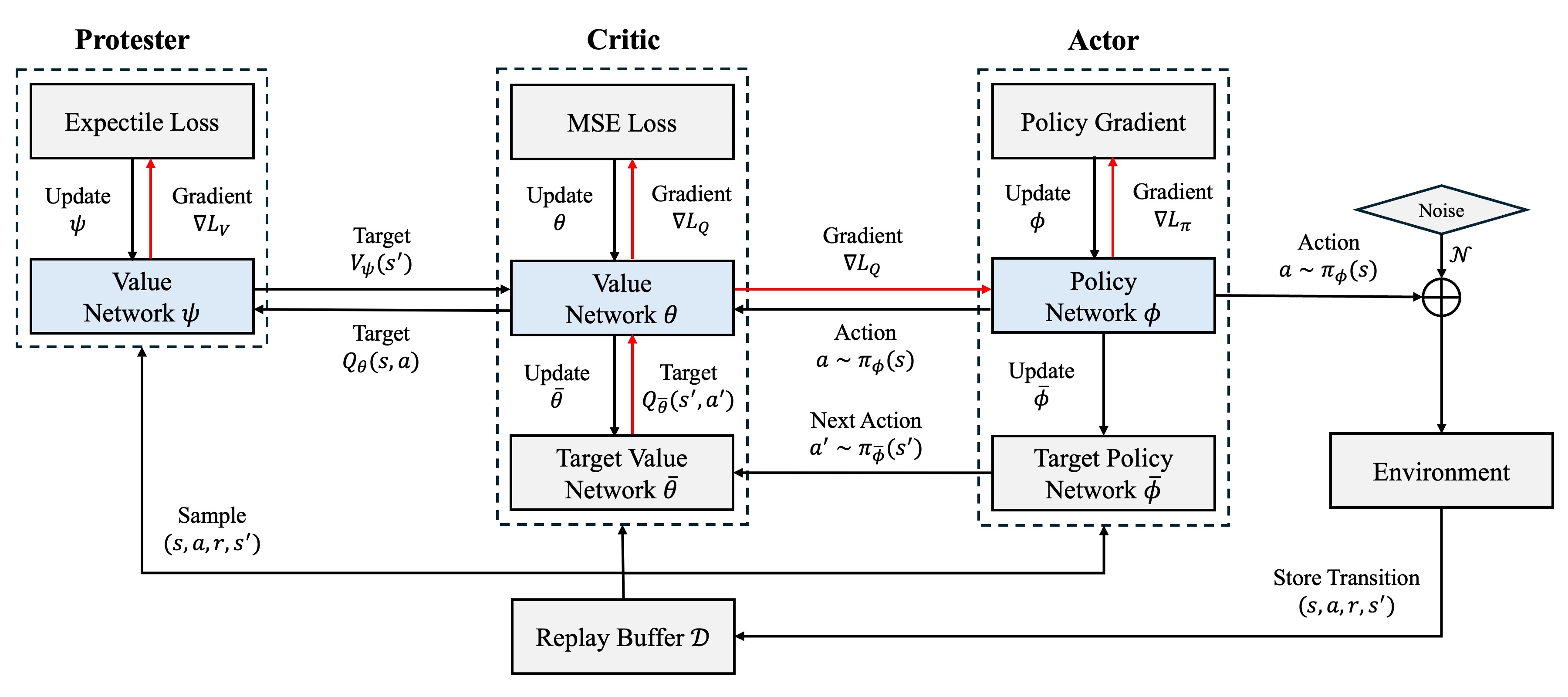} 
   \caption{Structure diagrams of algorithms utilizing the proposed protester.}
   \label{fig:structure_of_mpg}
\end{figure*}
%%%%%%%%%%%%%%%%%%%%%%%%%%%%%%%%%%%%

%%%%%%%%%%%%%%%%%%%%%%%%%%%%%%%%%%%%%%%%
\section{Introduction}\label{sec:intro}
%%%%%%%%%%%%%%%%%%%%%%%%%%%%%%%%%%%%%%%%%%%

Model-free reinforcement learning (MF-RL) has garnered considerable attention due to its capacity to learn without an explicit knowledge of an environment, thus proving highly applicable to complex and high-dimensional tasks \cite{arulkumaran2017deep, wang2022deep}. Recent advancements in deep neural networks (DNNs) have accelerated the practicality of MF-RL algorithms, enabling them to address increasingly challenging tasks across diverse domains including AI-native networks, autonomous systems, complex strategy optimization \cite{luong2019applications, yang2020deep, kiran2021deep}.

%Overestimation bias is a major obstacle when applying MF-RL algorithms to real-world applications \cite{thrun2014issues}. This issue is deeply rooted in the property of temporal difference (TD) learning \cite{thrun2014issues,pendrith1997estimator, mannor2007bias} and Bellman optimality equation \cite{bellman1966dynamic}. Specifically, the maximization of noisy Q-values in the standard TD target can induce a consistent overestimation \cite{thrun2014issues}. In continuous control environments, this issue becomes more serious as function approximations (e.g., DNNs) are bound to generate errors due to the imprecision of the estimator and insufficient training samples \cite{pendrith1997estimator, mannor2007bias}. 

In MF-RL algorithms \cite{mnih2015human, haarnoja2018soft_1, lillicrap2015continuous}, the accuracy of the Q-function approximation is critical in determining their stability and performances. The Q-function serves as the foundation for temporal difference (TD) target computation and action selection in value-based methods \cite{mnih2015human} or policy optimization in continuous actor-critic methods \cite{haarnoja2018soft_1, lillicrap2015continuous, kuznetsov2020controlling}. Overestimation bias is a major obstacle when applying MF-RL algorithms to real-world applications \cite{thrun2014issues}. 
This issue is deeply rooted in the properties of TD learning \cite{thrun2014issues,pendrith1997estimator, mannor2007bias} and Bellman optimality equation \cite{bellman1966dynamic}. Specifically, the maximization of noisy Q-values in the standard TD target can lead to a consistent overestimation \cite{thrun2014issues}.  
In the context of function approximations, this issue becomes more problematic, as the approximation noise is unavoidable due to the estimator's imprecision and an insufficient number of samples \cite{pendrith1997estimator, mannor2007bias}.

%RELATED WORKS 중복?
%{\BLUE Over the years, several approaches have been proposed to alleviate overestimation bias in  continuous control environments. Double Deep-Q-Network (DDQN) \cite{van2016deep} effectively mitigate this issue by decoupling action selection from Q-value estimation. However, \cite{fujimoto2018addressing} identified that DDQN is ineffective in the state-of-the-art (SOTA) actor-critic algorithms due to the slow-changing policy. To address this, the min-Q target was proposed by taking the minimum of two independent Q-values. Combining this target with Deep Deterministic Policy Gradient (DDPG) \cite{lillicrap2015continuous}, Twin-Delayed DDPG (TD3) \cite{fujimoto2018addressing} was proposed. Recently, 
%Truncated Quantile Critics (TQC) \cite{kuznetsov2020controlling}, implemented on top of Soft Actor Critic (SAC) \cite{haarnoja2018soft_1} and distributional RL, showed the best-known performance in large-scale environments by properly controlling overestimation bias.
%}

In this paper, we propose a novel {\em moderate} target designed to alleviate overestimation bias in RL. This target is formulated as a convex combination of an overestimated Q-function and its corresponding lower bound. Our key contribution lies in the effective derivation of the lower bound, achieved by estimating the lower expectile of the Q-value distribution conditioned on any given state. In this estimation, expectile loss \cite{bellini2017risk} is used as a generalization of the classical mean-squared-error (MSE) loss. By appropriately selecting the combination weight, the resulting Q-values in the moderate target can closely approximate the corresponding true Q-values. We incorporate the proposed moderate target into Deep Deterministic Policy Gradient (DDPG) \cite{lillicrap2015continuous}, Twin-Delayed DDPG (TD3) \cite{fujimoto2018addressing}, Soft Actor Critic (SAC) \cite{haarnoja2018soft_1}, and Truncated Quantile Critics (TQC) \cite{kuznetsov2020controlling}. This integration results in the development of new algorithms: Moderate Policy Gradient (MPG), Moderate Policy Gradient-Smoothing Delayed (MPG-SD), Moderate Actor-Critic (MAC), and Moderate Quantile Critics (MQC), respectively.
Due to the flexibility of the moderate target, notably, it can be seamlessly integrated with other RL methods, such as Q-learning \cite{watkins1992q} and DQN \cite{van2016deep} for discrete control tasks, as well as Advantage Actor-Critic (A2C) \cite{mnih2016asynchronous} and Proximal Policy Optimization (PPO) \cite{schulman2017proximal} for continuous control tasks.
%We incorporate the proposed moderate target into DDPG, TD3, SAC, and TQC, resulting in the development of Moderate Policy Gradient (MPG), Moderate Policy Gradient-Smoothing Delayed (MPG-SD), Moderate Actor-Critic (MAC), and Moderate Quantile Critics (MQC), which are designed for continuous control environments.
%Notably, due to the flexibility of the moderate target, it can be naturally integrated with other RL methods, such as Q-learning \cite{watkins1992q} and DQN \cite{van2016deep} for discrete control tasks, as well as Advantage Actor-Critic (A2C) \cite{mnih2016asynchronous} and Proximal Policy Optimization (PPO) \cite{schulman2017proximal} for continuous control tasks.
We validate the effectiveness of our algorithms through experiments on challenging continuous control tasks utilizing the MuJoCo physics engine \cite{todorov2012physics}, implemented in OpenAI Gym \cite{brockman2016openai}. Our algorithms consistently outperform their respective baseline counterparts in terms of both performance and variance (or stability). Importantly, this improvement is achieved without compromising training complexity. These results underscore the potential of the moderate target as a key component for the implementation of MF-RL algorithms across a wide range of complex environments.

%%%%%%%%%%%%%%%%%%%%%%%%%%%%%%%%%%%%%%%%%%%%%%%%%
\section{Related Works}\label{sec:related_works}

Numerous algorithms have been developed to mitigate overestimation bias. Double Q-learning \cite{hasselt2010double}, regarded as the de-facto algorithm for discrete control environments, employs two estimators of the Q-values. One estimator is used to select an action, while the other assesses the selected action, thereby ensuring that the Q-function is not overestimated. Double Deep-Q-Network (DDQN) \cite{van2016deep} extends this tabular-based approach to accommodate function approximations such as DNNs. In continuous control environments, \cite{fujimoto2018addressing} identified that DDQN is ineffective in popular actor-critic methods due to the slow-changing policy. In order to address this limitation, TD3 was introduced, which enhances DDPG \cite{lillicrap2015continuous} by incorporating a new target into the Q-value update. Specifically, overestimation bias can be mitigated by defining the target as the minimum value of two approximated Q-functions, referred to as the min-Q target. Due to its notable performance, this approach has been widely embraced and implemented in several subsequent works \cite{haarnoja2018soft_1, lan2020maxmin, hiraoka2021dropout, chen2021randomized}. In addition to enhancing the Q-value target, TD3 incorporated two further techniques to diminish the variance of estimates: delayed policy updates and target policy smoothing. 

Based on distributional RL \cite{bellemare2017distributional}, TQC \cite{kuznetsov2020controlling} was developed by blending three ideas: the distributional representation of a critic, the truncation of the approximated distribution, and ensembling. Instead of the traditional modeling of the Q-function, which relies on the expected return, TQC focuses on modeling the return distribution. By truncating the upper quantile of the estimated return distribution and leveraging the ensemble of multiple Q-value approximators, TQC effectively mitigates overestimation bias, thereby enhancing the robustness of Q-value estimates in large-scale environments.

In addition to the baseline algorithms, several variants have been introduced in the literature. Average DQN \cite{anschel2017averaged} diminishes the variance of target approximations by averaging previously learned Q-value estimates. Weighted double Q-learning \cite{zhang2017weighted} addresses the overestimation inherent in the standard Q-learning and the underestimation present in double Q-learning through a weighted combination. Furthermore, ensemble-based algorithms have been explored, including the use of a linear combination of the maximum and minimum Q-values from a pool of Q-networks \cite{li2019mixing, kumar2019stabilizing}, the aggregation of different Q-value predictions via a soft-max function \cite{pan2020softmax}, and the computation of Q-value targets through convex combinations of predictions from multiple policies \cite{lyu2022efficient}.

%문장표현 수정중!
%%%%%%%%%%%%%%%%%%%%%%%%%%%%%%%%%%%%%%%%%%%
\section{Background}\label{sec:prelimi}
%%%%%%%%%%%%%%%%%%%%%%%%%%%%%%%%%%%%%%%%%%%%

We provide the notations and definitions, and briefly explain the baseline frameworks to build our algorithms.

%%%%%%%%%%%%%%%%%%%%%%%%%%%%%%%%%%%%%%%%%
\subsection{Model-Free RL}
%%%%%%%%%%%%%%%%%%%%%%%%%%%%%%%%%%%%%%%%

We consider an infinite-horizon Markov Decision Process (MDP) \cite{watkins1992q, puterman2014markov}, defined by the tuple $(\mathcal{S}, \mathcal{A}, p, r, \gamma)$, where $\Sc$ is a state space, $\Ac$ is an action space,  $p : \mathcal{S} \times \mathcal{S} \times \mathcal{A} \rightarrow [0, \infty)$ is unknown state transition distribution, $r: \mathcal{S} \times \mathcal{A} \rightarrow \mathbb{R}$ is a reward function, and $\gamma \in (0, 1)$ is a discount factor.
At every time step $t$, an agent observes the current state $s_t \in \mathcal{S}$, selects an action $a_t \in \mathcal{A}$ according to its policy $\pi(a_t \mid s_t)$, and receives the reward $r(s_t, a_t)$. The environment transitions to a new state $s_{t+1}$ according to $p(s_{t+1}\mid s_t, a_t)$. The agent aims at seeking an optimal policy, denoted as $\pi^{\star}$, to maximize the expected return:
\begin{equation}
     J(\pi) = \mathbb{E}_{\pi}\left[ \sum_{t=0}^{\infty} \gamma^t r(s_t, a_t) \right],
    \label{eq:policy_objective}
\end{equation} where the expectation is taken over the state-action trajectories generated by a  policy $\pi$. To find an optimal policy, it is crucial to accurately evaluate the actions taken in a given state. This requires the precise estimation of the action-value function, known as the Q-function, which predicts the cumulative future rewards for a given state-action pair. The Q-function is formally defined as:
\begin{equation}
    Q^\pi(s, a) = \mathbb{E}_{\pi} \left[ \sum_{t=0}^{\infty} \gamma^t r(s_t, a_t)\; \Bigg\vert\; s_0 = s, a_o = a \right].\label{eq:Q-function}
\end{equation}
To address the optimization problem in continuous control environments, the actor-critic method based on function approximations is widely used. The actor and critic correspond to the policy and Q-function, respectively. In this paper, the actor and critic are represented using parameterized functions denoted as  $\pi_{\phi}$ and $Q_{\theta}$, respectively.

%%%%%%%%%%%%%%%%%%%%%%%%%%%%%%%%%%
\subsection{Overestimation Bias}
%%%%%%%%%%%%%%%%%%%%%%%%%%%%%%%%%%

In Q-learning \cite{watkins1992q}, the Q-function is learned using the greedy target:
\begin{equation}
    y = r(s_t, a_t) + \max_{a'\in \Ac}\; Q^{\pi}(s_{t+1}, a'). \label{eq:greedy_target}
\end{equation}
\citet{thrun2014issues} showed that for an implicit true Q-function $Q^{\star}=Q^{\pi^{\star}}$, the estimated function $Q_{\theta}$ might exhibit overestimation bias, as outlined below:
\begin{align}
    \max_{a'\in\Ac}\; Q^{\star}(s_{t+1},a') &\stackrel{(a)}{=}  \max_{a'\in\Ac}\; \mathbb{E}\left[{Q}_{\theta}(s_{t+1}, a') \right]\nonumber\\
    &\stackrel{(b)}{\leq} \mathbb{E} \left[ \max_{a'\in\Ac}\; {Q}_{\theta}(s_{t+1}, a') \right],\label{eq:OB}
\end{align} where $Q_{\theta}(s,a)$ is a random variable as a function of random samples, (a)  is due to the fact that $Q_{\theta}$ is an unbiased estimator of $Q^{\star}$ (i.e., $\EE[Q_{\theta}]=Q^{\star}$), and (b) follows the Jensen's inequality. Consequently, this overestimation bias can propagate through the Bellman equation. In actor-critic methods, it is therefore essential to modify the greedy target in Equation~\ref{eq:greedy_target} to mitigate overestimation bias.

DDQN \cite{van2016deep} utilizes the double-critic approach, in which one critic is used to choose a greedy action and the other critic is used to evaluate the action. Based on this, the double-Q target is defined as follows:
\begin{gather}
    y_{\text{duo}} = r(s_{t}, a_{t}) +  Q_{\theta_{1}}(s_{t+1}, a' )\nonumber\\
    a' = \argmax_{a'\in \Ac} Q_{\theta_2}(s_{t+1}, a'). \label{eq:double_target}
\end{gather} 
However, \cite{fujimoto2018addressing} identified that DDQN does not fully address overestimation bias in continuous control environments. As an alternative, the authors proposed Clipped Double Q-learning, where the min-Q target is defined as
\begin{gather}
    y_{\text{min}} = r(s_{t}, a_{t}) +  \min_{i \in \{1,2\}} Q_{\theta_i}(s_{t+1}, a')\nonumber\\
    % a'\sim\pi_{\phi}(\cdot \mid s_{t+1}).
    a'=\pi_{\phi}( s_{t+1}).
    \label{eq:clipped_target}
\end{gather} The min-Q target can be directly used in continuous control environments. In Section~\ref{subsec:protester}, we will propose a novel target that more effectively mitigates overestimation bias than the aforementioned double-Q and min-Q targets.

%%%%%%%%%%%%%%%%%%%%%%%%%%%%%%%%%%
\subsection{Actor-Critic Methods}
%%%%%%%%%%%%%%%%%%%%%%%%%%%%%%%%%%

We review DDPG \cite{lillicrap2015continuous}, SAC  \cite{haarnoja2018soft_1}, and TQC \cite{kuznetsov2020controlling} since they will serve as the baseline actor-critic methods of our algorithms. Throughout the paper, $\Dc=\{(s=s_t,a=a_t,r=r(s_t,a_t),s'=s_{t+1})\}$ represents the replay buffer containing  samples. %i.e., $\Dc=\{(s=s_t,a=a_t,r=r(s_t,a_t),s'=s_{t+1})\}$.

%%%%%%%%%%%%%%%%%%%%%%%%
\subsubsection{DDPG}
%%%%%%%%%%%%%%%%%%%%%%%%

DDPG enhances the deterministic policy gradient (DPG) method \cite{silver2014deterministic} by utilizing DNNs to effectively handle large-scale environments. Furthermore, it operates in an off-policy manner,  employing the replay buffer and  incorporating the target actor and critic networks to ensure stable training \cite{mnih2013playing}. The critic network is trained by minimizing the temporal difference (TD) error, wherein the loss function is defined as
\begin{equation}
    L_{Q}(\theta) = \mathbb{E}_{(s,a,r, s') \sim \mathcal{D}} \left[ (y - Q_{\theta}(s, a))^2 \right],
    \label{eq:ddpg_q_update}
\end{equation}
and from the Bellman equation, the standard target $y$ is determined by
\begin{equation}
    y = r + \gamma Q_{\bar{\theta}}(s', a'),\;\;\;\;\; a'=\pi_{\bar{\phi}}(s').
    \label{eq:ddpg_q_target}
\end{equation} Herein, $\bar{\theta}$ and  $\bar{\phi}$ denote the parameters of the target critic and actor networks, respectively. This target represents an extension of the greedy target in Equation~\ref{eq:greedy_target}, specially adapted for continuous control settings. The actor network is trained by maximizing the expected Q-value for action selection, with the loss function:
\begin{equation}
    L_\pi(\phi) = \mathbb{E}_{s \sim \mathcal{D}} \left[ - Q_{\theta}(s, \pi_{\phi}(s)) \right].
    \label{eq:ddpg_policy_update}
\end{equation}
The actor and critic networks are updated alternatively, and the target networks are updated via soft-update mechanism with a target update rate $\eta \in (0.1)$:
\begin{align}
     \bar{\phi} \leftarrow \eta \phi + (1 - \eta) \bar{\phi}, \quad \bar{\theta} \leftarrow \eta \theta + (1 - \eta) \bar{\theta}. 
    \label{eq:target_update}
\end{align} 

%To compare the min-Q target with our target in Section~\ref{subsec:protester}, 
In addition, we develop DDPG(min-Q) by replacing the standard target with the min-Q target below:
%Herein, the target in Equation~\ref{eq:ddpg_q_target} is replaced with:
\begin{equation}
    y_{\rm min} = r + \gamma \min_{i\in\{1,2\}}Q_{\bar{\theta}_i}(s', a'), \;\; a' = \pi_{\bar{\phi}}(s').
    \label{eq:min_Q_target_for_ddpg}
\end{equation}
Following the approach in TD3 \cite{fujimoto2018addressing}, the actor network is learned with following loss function:
\begin{equation}
    L_\pi(\phi) = \mathbb{E}_{s \sim \mathcal{D}} \left[ - Q_{\theta_1}(s, \pi_{\phi}(s)) \right].
    \label{eq:ddpg_min_q_policy_update}
\end{equation}

%%%%%%%%%%%%%%%%%%%%
\subsubsection{SAC}
%%%%%%%%%%%%%%%%%%%%

To enhance exploration, SAC \cite{ziebart2010modeling, haarnoja2018soft_1, haarnoja2018soft_2} utilizes the maximum entropy objective:
\begin{equation}
     J(\pi) = \mathbb{E}_{\pi}\left[ \sum_{t=0}^{\infty} \gamma^t \big( r(s_t, a_t) + \alpha \mathcal{H}(\pi(\cdot \mid s_t)) \big) \right],
    \label{eq:max_entropy_objective}
\end{equation} where $\alpha$ is the temperature parameter that controls the relative importance of the entropy term in relation to the reward and $\mathcal{H}(p)$ denotes the entropy of a distribution $p$. 
SAC typically employs the min-Q target to reduce overestimation bias. Then, the loss function for the actor is defined as
\begin{align}
    L_\pi(\phi) &= \mathbb{E}_{s \sim \mathcal{D}, \epsilon \sim \mathcal{N}(0, 1)} \bigg[ \alpha \log \pi_{\phi}(f_{\phi}(\epsilon ; s)\mid s)  \nonumber\\
    &\quad\quad\quad - \min_{i \in \{ 1, 2 \}} Q_{\theta_{i}}(s, f_{\phi}(\epsilon ; s)) \bigg], 
    \label{eq:sac_policy_update}
\end{align} where $f_{\phi}(\epsilon; s)$ denotes the reparameterization function to sample an action. The loss function for the critic is same as that of DDPG in Equation \ref{eq:ddpg_q_update}. However, the target $y$ is redefined by combining the maximum entropy approach and the min-Q target:
\begin{gather}
    y = r  + \gamma \left[ \min_{i \in \{ 1, 2 \}} Q_{\bar{\theta}_i}(s', a') - \alpha \log \pi_{\phi}(a' \mid s') \right],\nonumber \\ 
    a' \sim \pi_{\phi}( \cdot \mid s').
    \label{eq:sac_target}
\end{gather}  
The parameters of the target critic networks $\bar{\theta}_i$'s are updated using a soft update mechanism. Regarding the temperature parameter $\alpha$, it is adaptively chosen to keep the desired level of the entropy in the policy \cite{haarnoja2018soft_2}. To emphasize the use of min-Q target, throughout the paper, SAC is referred to as SAC(min-Q).

%%%%%%%%%%%%%%%%%%%%
\subsubsection{TQC}
%%%%%%%%%%%%%%%%%%%%%

Building on the distributional perspective \cite{bellemare2017distributional}, \citet{kuznetsov2020controlling} proposed TQC by extending QR-DQN \cite{dabney2018distributional}, originally developed for discrete control, to continuous control. In this extension, TQC incorporates the maximum entropy framework (e.g., SAC) \cite{haarnoja2018soft_2} to further enhance the performance in continuous control tasks. For any positive integer $N$, we let $[N]:=\{1,2,...,N\}$.

Distributional RL aims to estimate the distribution of a random return $Z^\pi(s, a)=\sum_{t=0}^{\infty} \gamma^t r(s_t, a_t)$, where $s_0=s, a_0=a$ and $s_{t+1}\sim p(\cdot|s_t,a_t), a_t \sim \pi(\cdot|s_t)$, instead of the Q-function $Q^\pi(s, a) = \mathbb{E}[Z^\pi(s, a)]$. To this end, TQC defines the $N$ critic networks parameterized by $\{\theta_n: n \in [N]\}$, each of which generates $M$ atoms $\{\kappa_{\theta_{n}}^{m}(s,a): m \in [M]\}$ for an action-state pair $(s,a)$. Using them, the $N$ approximators of the distribution of $Z^\pi(s, a)$ are defined as
\begin{equation}
    Z_{\theta_n}(s, a) = \frac{1}{M} \sum_{m=1}^M \delta\left( \kappa_{\theta_n}^m(s, a) \right),\; n \in [N],
    \label{eq:tqc_q_distribution}
\end{equation} where $ \delta(\cdot)$ denotes the Dirac delta function. To learn the critic network, the target distribution is defined as 
\begin{gather}
    Y(s, a) = \frac{1}{kN} \sum_{i = 1}^{kN} \delta(y_i(s, a)),
    \label{eq:tqc_target_distribution}
\end{gather} where $k$ represents the number of atoms selected per network, and $y_i$ is defined as:
\begin{gather}
     y_i(s, a) = r + \gamma \left[ z_{(i)}(s', a') - \alpha \log \pi_{\phi}(a' \mid s') \right]\nonumber\\
     a' \sim \pi_{\phi}( \cdot \mid s').
\end{gather} Herein, $z_{(i)}(s', a')$ is the $i$-th smallest element in the following atom set:
\begin{equation}
    \mathcal{Z}(s', a') = \left\{ \kappa_{\bar{\theta}_n}^{m}(s', a') \mid 
    m\in [M], n\in [N] 
    \right\},
    \label{eq:tqc_atom_set}
\end{equation} where $\{\bar{\theta}_{n}: n \in [N]\}$ denote the parameters of the target critic networks which are updated via soft update mechanism. The parameters $\{\theta_{n}: n \in [N]\}$ of the critic networks are updated by using the target distribution in Equation~\ref{eq:tqc_target_distribution} and the Huber quantile loss function \cite{huber1992robust, dabney2018distributional, kuznetsov2020controlling}. Also, to learn the actor network, the loss function in Equation~\ref{eq:sac_policy_update} is modified as
\begin{gather}
    L_\pi(\phi) = \mathbb{E}_{s \sim \mathcal{D}, \epsilon \sim \mathcal{N}(0, 1)} \bigg[ \alpha \log \pi_{\phi}(f_{\phi}(\epsilon ; s) \mid s)  \nonumber \\
    \qquad \qquad \qquad - \frac{1}{MN} \sum_{n = 1}^{N}\sum_{m=1}^{M} \kappa_{\theta_n}^m(s , f_{\phi}(\epsilon ; s)) \bigg]. 
    \label{eq:tqc_policy_loss}
\end{gather}

% %%%%%%%%%%%%%%%%%%%%%%%%%%%%%%%%%%%%%%%%%%%%
% \begin{figure}[t]
%     \centering  
%     \includegraphics[width=0.99\linewidth]{plots/expectile_loss.png} 
%     \caption{The expectile loss function at various expectile levels.}
%     \label{fig:expectile_loss_function}
% \end{figure}
% %%%%%%%%%%%%%%%%%%%%%%%%%%%%%%%%%%%%

%%%%%%%%%%%%%%%%%%%%%%%%%%%%%%%%%%%%%%%%%%%%
\begin{figure}[t]
    \centering  
    \includegraphics[width=0.99\linewidth]{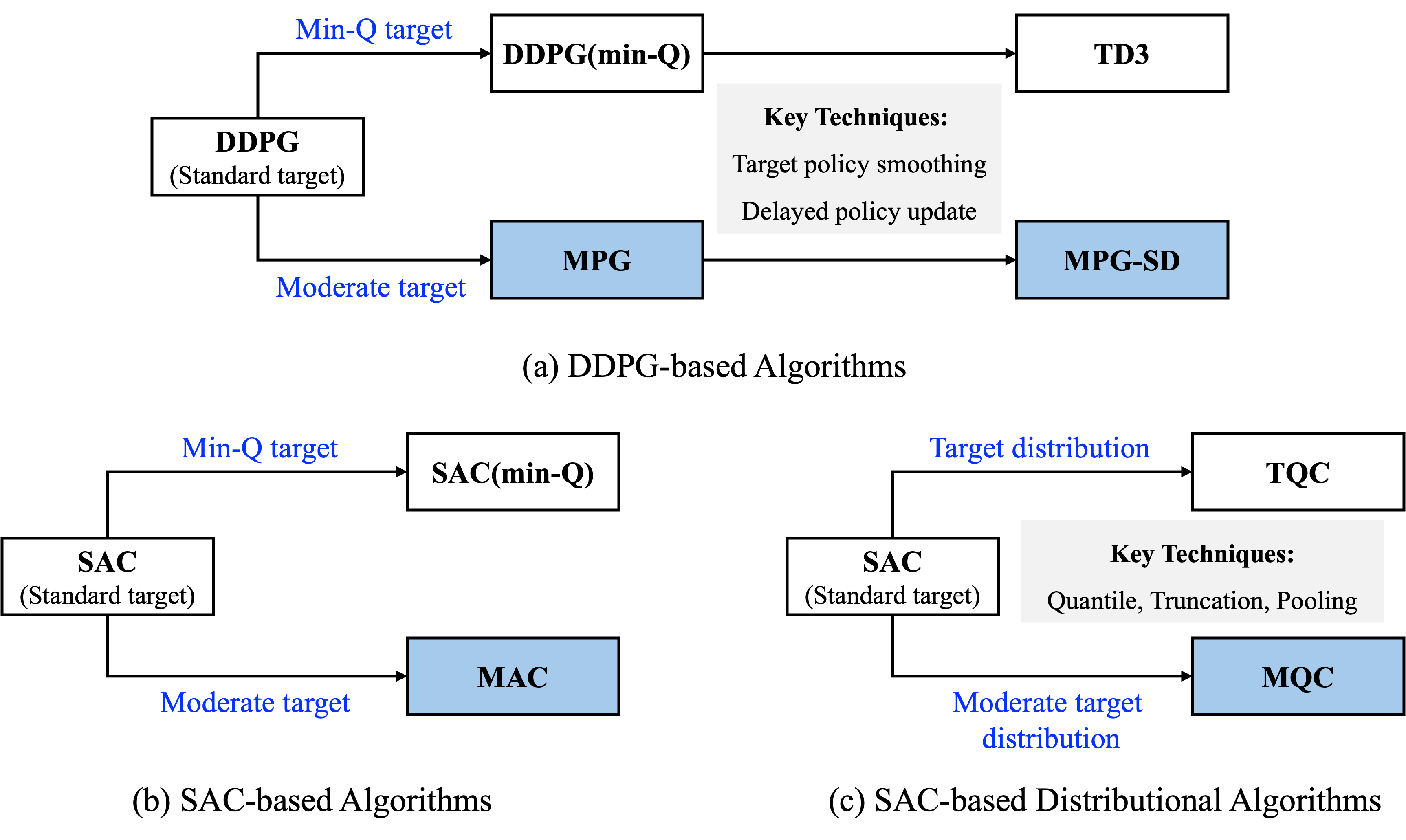} 
    \caption{Relationships between the proposed and benchmark algorithms, where the shaded boxes represent our algorithms.}
    \label{fig:algorithms}
\end{figure}
%%%%%%%%%%%%%%%%%%%%%%%%%%%%%%%%%%%%

%%%%%%%%%%%%%%%%%%%%%%%%%%%%%%%%%%%%%%%%%%
\section{Algorithms}\label{sec:algorithm}
%%%%%%%%%%%%%%%%%%%%%%%%%%%%%%%%%%%%%%%%%%

We first present a novel moderate target designed to effectively mitigate overestimation bias. Integrating this target into DDPG and SAC, we establish {\bf M}oderate {\bf P}olicy {\bf G}radient ({\bf MPG}) and {\bf M}oderate {\bf A}ctor {\bf G}radient ({\bf MAC}), respectively. In addition, we apply our moderate target to the state-of-the-art (SOTA) distributional RL, dubbed TQC, resulting in the development of {\bf M}oderate {\bf Q}uantile {\bf C}ritics ({\bf MQC}). The detailed procedures of the proposed algorithms are provided in the supplementary material.

\begin{figure*}
    \centering
    \begin{subfigure}[b]{0.195\textwidth}
         \centering
         \includegraphics[width=\textwidth]{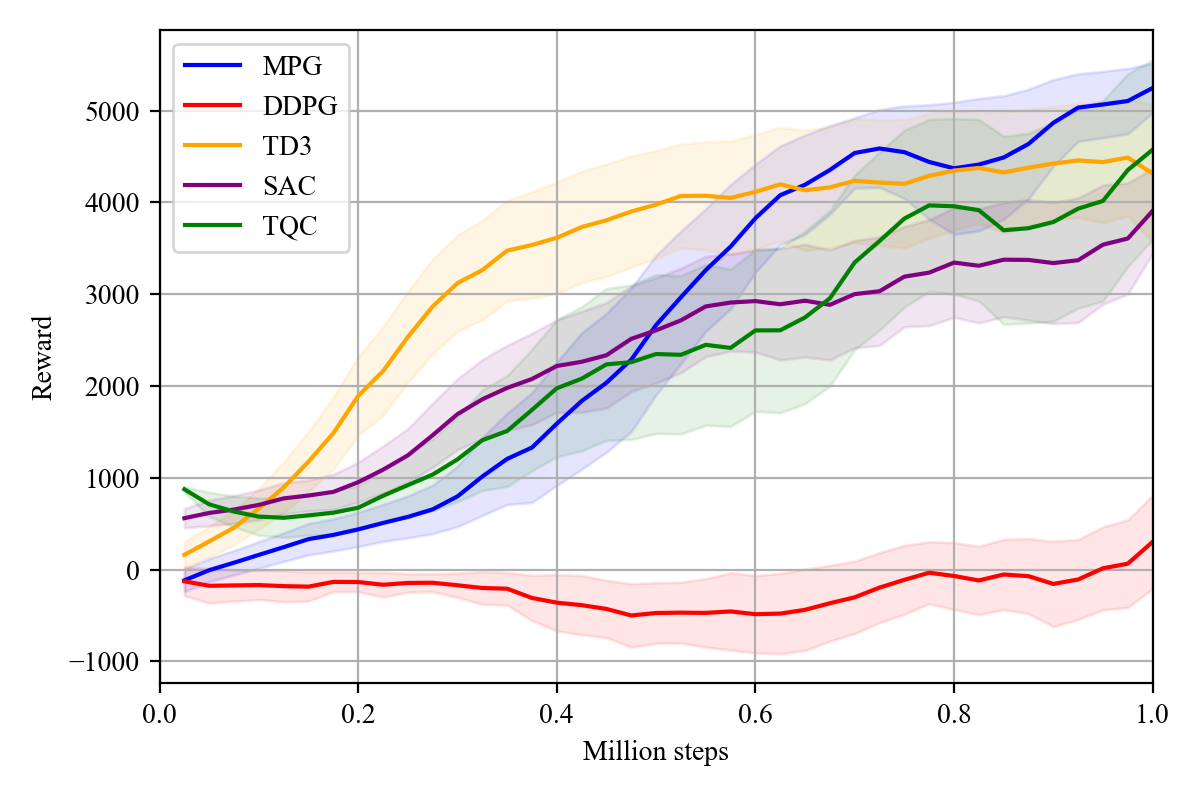}
         \caption{Ant-v4}
    \end{subfigure}
    \hfill
    \begin{subfigure}[b]{0.195\textwidth}
         \centering
         \includegraphics[width=\textwidth]{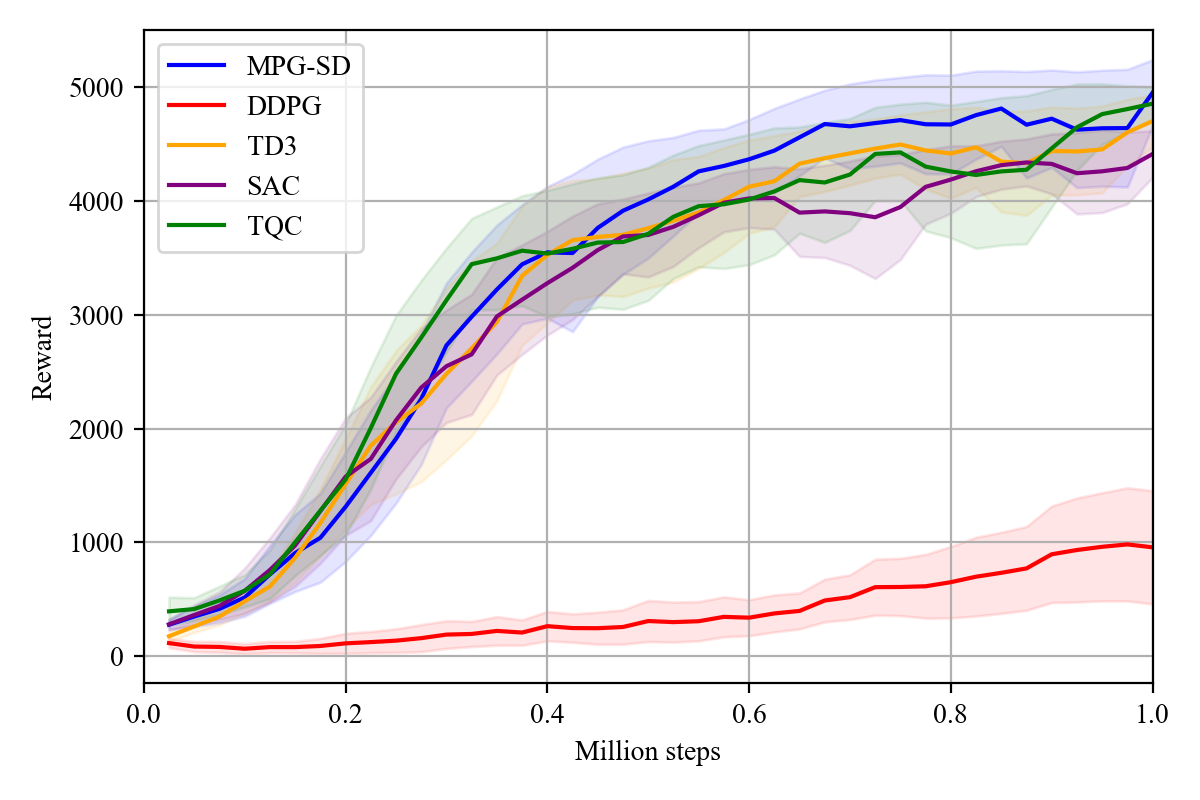}
         \caption{Walker2d-v4}
    \end{subfigure}
    \hfill
    \begin{subfigure}[b]{0.195\textwidth}
         \centering
         \includegraphics[width=\textwidth]{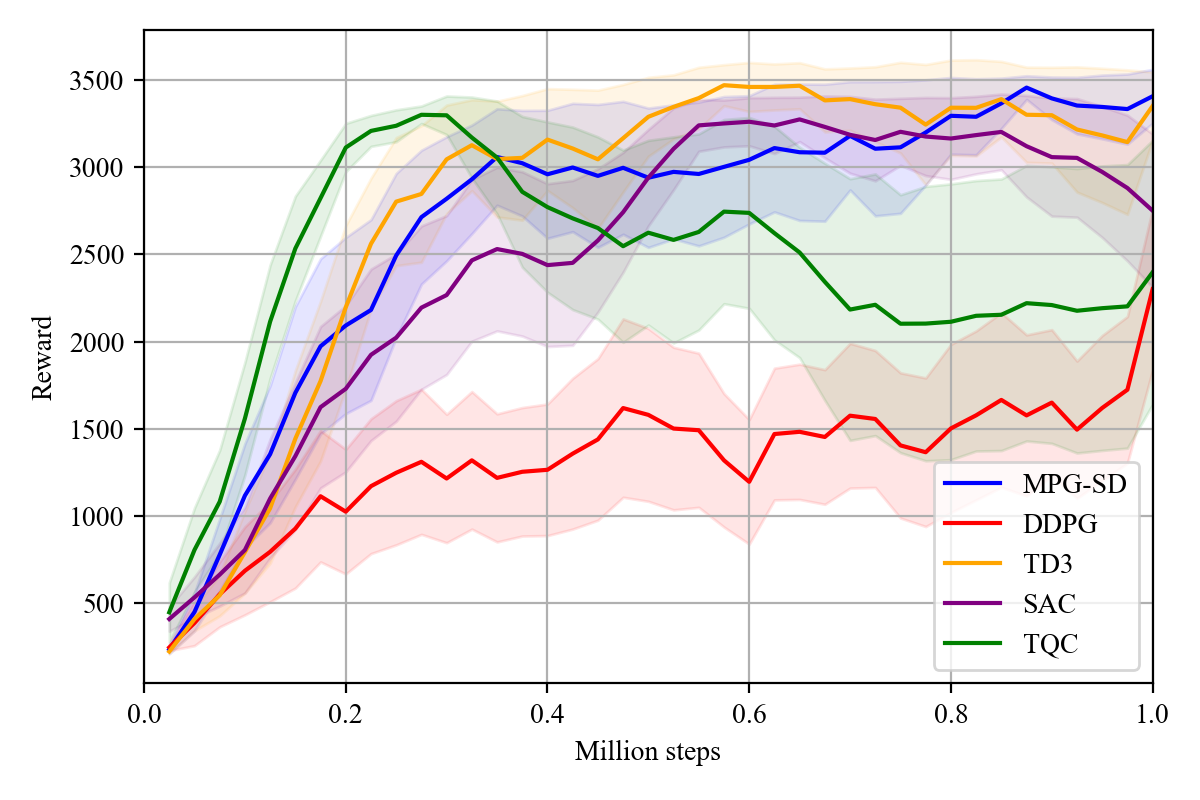}
         \caption{Hopper-v4}
    \end{subfigure}
    \hfill
    \begin{subfigure}[b]{0.195\textwidth}
         \centering
         \includegraphics[width=\textwidth]{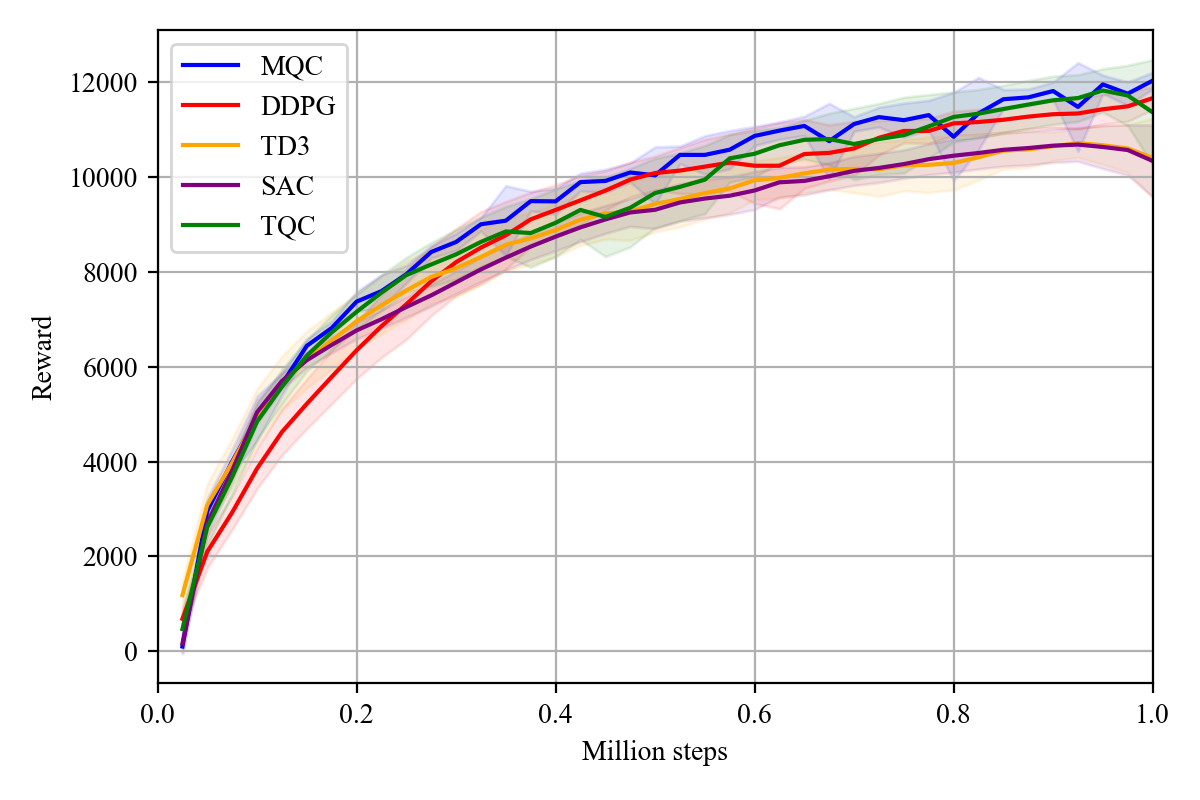}
         \caption{HalfCheeetah-v4}
    \end{subfigure}
    \hfill
    \begin{subfigure}[b]{0.195\textwidth}
         \centering
         \includegraphics[width=\textwidth]{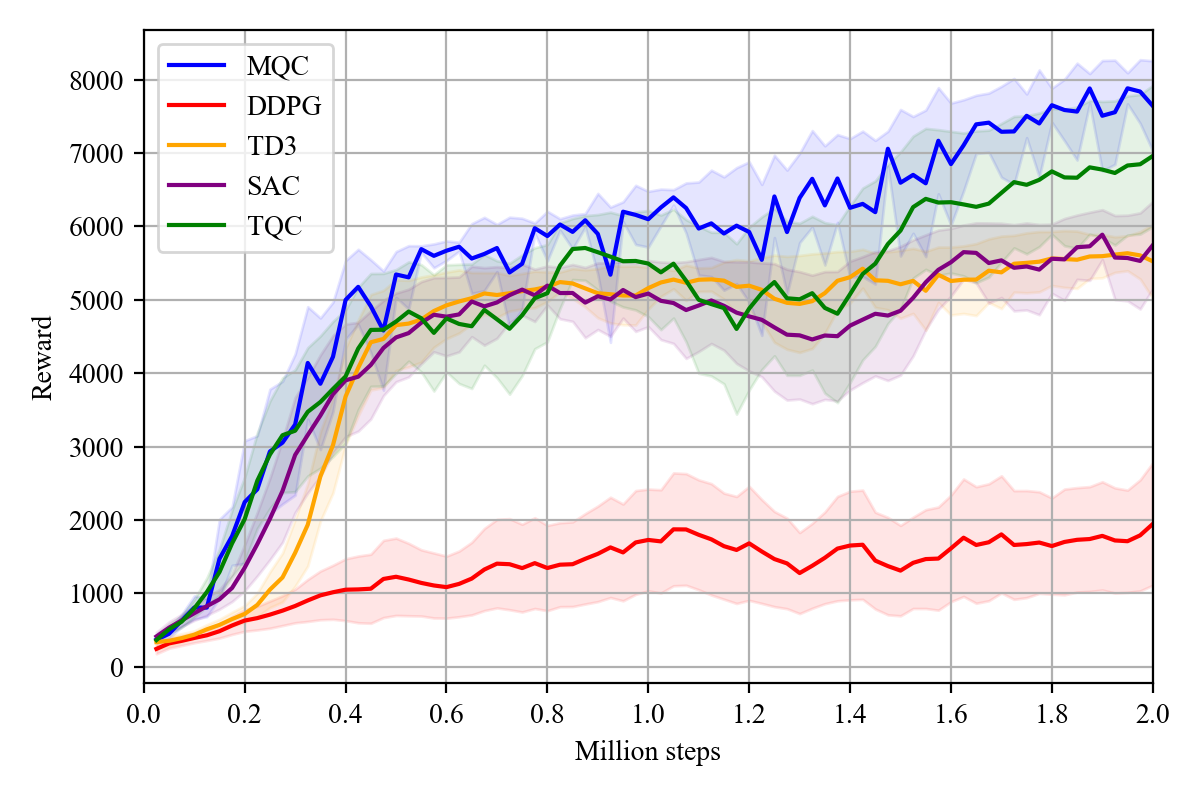}
         \caption{Humanoid-v4}
    \end{subfigure}
    \caption{Learning curves for MuJoCo continuous control tasks. The solid lines denote the average rewards and the shaded areas indicate half the standard deviation of the average evaluations over five episodes. Curves are smoothed with a moving average window for clarity.}
    \label{fig:simulation_graph}
\end{figure*}
%for visual clarity

%%%%%%%%%%%%%%%%%%%%%%%%%%%%%%%%%%%%%%%%%%%%%%%
\subsection{Protester and Moderate Target}\label{subsec:protester}
%%%%%%%%%%%%%%%%%%%%%%%%%%%%%%%%%%%%%%%%%%%%%%%

An expectile \cite{newey1987asymmetric} is the notion to describe the distribution of a random variable. The expectile at level $\tau \in (0,1)$ is defined as the minimizer of the asymmetrically weighted squared-error loss function:
\begin{equation}
\hat{y}=\argmin_{x}\EE\left[\ell_\tau(Y,x)\right],\label{opt:expectile}
\end{equation} where the expectation is taken over the distribution of the random variable $Y$ and
\begin{equation}
    \ell_\tau(y, x) = 
    \begin{cases}
        \tau(y - x)^2          & \text{if} \quad y \ge x    \\
        (1-\tau)(y - x)^2      & \text{if} \quad y < x   
    \end{cases}.
    \label{eq:expectile_loss}
\end{equation} This loss function is the basis for the expectile regression, generalizing the classical linear regression in terms of the predicted variable distribution. For $\tau = 0.5$, the loss function in Equation~\ref{opt:expectile} corresponds to the standard mean squared error (MSE) (equivalently, $\hat{y}=\EE[Y]$). For $\tau < 0.5$ and $\tau > 0.5$, on the other hand, it becomes more flexible, placing greater emphasis on the lower or upper tail of the distribution, respectively. 
% Figure~\ref{fig:expectile_loss_function} depicts the loss function in Equation~\ref{eq:expectile_loss} at various expectile levels. 
Consequently, expectile regression can be used to estimate the lower or upper distribution of the Q-values according to the choices of the expectile level $\tau$.

By leverage the expectile at level $\tau \in (0,1)$, we introduce a generalized state-value function:
\begin{equation}
    V^{\pi}_\tau(s) = \argmin_{v} \EE_{a \sim \pi(\cdot|s)}\left[\ell_{\tau}\left(Q^{\pi}(s,a), v \right) \right].\label{eq:e-loss}
\end{equation} 
When $\tau=0.5$, it is equivalent to the standard state-value function, i.e., $V_{\tau=0.5}^{\pi}(s)=\EE_{\pi}[Q^{\pi}(s,a)]$. By controlling the expectile level $\tau$, we can estimate the various aspects of the Q-value distribution conditioned on the state $s$. To control overestimation bias, we will use the $V^{\pi}_\tau(s)$ with a sufficiently small $\tau$ as follows. As shown in Equation~\ref{eq:OB}, overestimation bias occurs due to the fact that %expectation추가
\begin{gather}
     Q^{\star}(s_{t+1}, a_{t+1}^{\star})  \leq \EE \left[ \max_{a'\in\Ac}\; {Q}_{\theta}(s_{t+1}, a')\right]\nonumber\\
      a_{t+1}^{\star} = \argmax_{a' \in \mathcal{A}}\; Q^{\star}(s_{t+1}, a'). \label{eq:up}
\end{gather}
With a sufficiently small $\tau$ (e.g., $\tau=10^{-2}$), it is highly likely that
\begin{gather}
 V_{\tau=10^{-2}}(s_{t+1}) 
 <  Q^{\star}(s_{t+1}, a_{t+1}^{\star})\nonumber\\
  V_{\tau}(s) = \argmin_{v}\; \EE_{a \sim \pi(\cdot|s)}\left[\ell_{\tau}\left(Q_{\theta}(s,a), v \right) \right]. \label{eq:lo}
\end{gather} 
Then, there certainly exists $\omega\in[0,1]$ such that
% {\RED 
\begin{align}
    &Q^{\star}(s_{t+1}, a_{t+1}^{\star})\nonumber\\
    &= (1-\omega)\EE\left[\max_{a'\in\Ac} {Q}_{\theta}(s_{t+1}, a')\right]+\omega V_{\tau=10^{-2}}(s_{t+1}).
\end{align}
% } 
By appropriately selecting a cautious weight $\omega \in [0,1]$, we can successfully address overestimation bias.
As depicted in Figure 1,  we propose an expectile value network, designated as {\bf protester}, to effectively incorporate this concept within the actor-critic network. Throughout the paper, it is represented as the parameterized function $V_{\psi}(\cdot)$.
From Equation~\ref{eq:lo}, the parameter $\psi$ is optimized using the {\em expectile} loss function:
\begin{equation}
    L_{V}(\psi) = \mathbb{E}_{(s, a) \sim \mathcal{D}} \left[ \ell_\tau\left(Q_\theta(s, a), V_\psi(s)\right) \right].
    \label{eq:proposed_expectile_value_loss}
\end{equation}

To mitigate overestimation bias, we present a moderate target, which is defined as a convex combination of the standard target Q-value (i.e., $Q_{\bar{\theta}}$) and the expectile value (i.e., $V_{\psi}$). Note that the former is an overestimated Q-function and the latter is the lower bound of the Q-function. Thus, properly selecting the combination weight, the resulting Q-values in the moderate target can closely approximate the corresponding true Q-values. The proposed moderate target is defined as
\begin{gather}
    y_{\rm mt} = r + \gamma \left[ (1 - \omega) Q_{\bar{\theta}}(s', a') + \omega  V_{\psi}(s') \right] \nonumber \\
    a' = \pi_{\bar{\phi}}(s'), \label{eq:moderate_target}
\end{gather} where $\bar{\phi}$ and $\bar{\theta}$ denote the parameters of the target actor and critic networks, respectively, and $\psi$ denotes the parameter for the protester.
The hyperparameter $\omega \in [0,1]$ represents a cautious weight that is determined according to the degree of overestimation bias. For instance, selecting a larger $\omega$ results in a more cautious target, which is particularly suitable for environments with severe overestimation bias. Therefore, our moderate target refines the standard target in Equation~\ref{eq:ddpg_q_target}, effectively alleviating the overestimation of Q-value estimates.

%%%%%%%%%%%%%%%%%%%%%%%%%%%%%%%%%%%%%%%%%%%%%
\subsection{Moderate Policy Gradient (MPG)} %
%%%%%%%%%%%%%%%%%%%%%%%%%%%%%%%%%%%%%%%%%%%%%

We describe the proposed MPG, which is developed by incorporating the moderate target into DDPG. Leveraging the moderate target $y_{\rm mt}$ in Equation~\ref{eq:moderate_target}, the critic network is optimized with the loss function:
\begin{equation}
    L_{Q}(\theta) = \mathbb{E}_{(s, a, r, s') \sim \mathcal{D}} \left[ \left( y_{\rm mt} - Q_\theta(s, a) \right)^2 \right].
    \label{eq:proposed_critic_loss}
\end{equation} 
It is important to note that, as the moderate target employs the lower expectile of the Q-value distribution, it is generally less than the standard target. Consequently, the estimated Q-value in our MPG tends to be lower than that in DDPG, thus enabling a more stable decision-making policy. The actor network is optimized exactly following the procedures of DDPG, as outlined in  Equation~\ref{eq:ddpg_policy_update}. Nonetheless, since our estimated Q-value is more conservative than that in DDPG, the actor network is trained to make a decision in a more cautious way.
The target critic and actor networks are updated via a soft update mechanism, as described in Equation~\ref{eq:target_update}.

To further mitigate overestimation bias, MPG is combined with the variance-reduction techniques in TD3 \cite{fujimoto2018addressing}, including target policy smoothing and delayed policy updates. The actor, target actor, and target critic networks are updated at every $d\geq 1$ steps. Also, the moderated target is modified by adding a clipped random noise:
\begin{gather}
    y_{\rm mt} = r + \gamma \left[ (1 - \omega) Q_{\bar{\theta}}(s', a') + \omega V_{\psi}(s') \right] \nonumber \\
    a' = \pi_{\bar{\phi}}(s') + \bar{\epsilon}, \label{eq:moderate_target_for_td3}
\end{gather} where $\bar{\epsilon} \sim \mbox{clip}(\Nc(0,\sigma^2),-c,c)$. The resulting algorithm is named MPG-SD. For $d=1$ and $\bar{\epsilon}=0$, it reduces to MPG.

%%%%%%%%%%%%%%%%%%%%%%%%%%%%%%%%%%%%%%%%%
\subsection{Moderate Actor Critic (MAC)}

We describe the proposed MAC, which is constructed by replacing the min-Q target in SAC(min-Q) with our moderate target. To this end, one of the two critics in SAC(min-Q) is changed into the protester, described in Equation~\ref{eq:proposed_expectile_value_loss}. Applying the maximum entropy object to the moderate target in Equation~\ref{eq:moderate_target}, the target in MAC is defined as
\begin{gather}
    y_{\rm mt} = r  + \gamma \Big[ (1 - \omega) Q_{\bar{\theta}}(s', a') + \omega V_{\psi}(s')\nonumber\\
    \qquad \qquad \;\;\;\;\; - \alpha \log \pi_{\phi}(a' \mid s') \Big], \;\;\; a' \sim \pi_{\phi}( \cdot \mid s').\label{eq:mac_target}
\end{gather}  Since MAC uses the single critic, the loss function of the actor is redefined as
\begin{align}
    L_\pi(\phi)&= \mathbb{E}_{s \sim \mathcal{D}, \epsilon \sim \mathcal{N}(0, 1)} \Big[ \alpha \log \pi_{\phi}(f_{\phi}(\epsilon; s) \mid s)\nonumber\\
    &\quad\quad\quad\quad\quad\quad\quad\quad\quad\quad - Q_{\theta}(s, f_{\phi}(\epsilon; s)) \Big].\label{eq:mac_policy_update} 
\end{align}

%%%%%%%%%%%%%%%%%%%%%%%%%%%%
%  TABLE
%%%%%%%%%%%%%%%%%%%%%%%%%%%%
\begin{table*}[t]
    \centering
     \caption{Average reward and standard deviation calculated after training over five episodes for each algorithm, with training conducted across five different seeds. The maximum value for each task is highlighted in bold.
     }
    \label{table:simulation_performance_result}
    \renewcommand{\arraystretch}{1.2}
        \resizebox{\textwidth}{!}{
        {\small
        \begin{tabular}{ c c c c c c c c} 
        \toprule
            & \bfseries Ant-v4 & \bfseries Walker2d-v4 & \bfseries Hopper-v4  & \bfseries HalfCheetah-v4 & \bfseries Humanoid-v4 & \bfseries Average \\ \midrule
         DDPG       &  302.8 ± 1021.1   & 956.3 ± 996.0 & 2300.4 ± 924.5 & 11660.5 ± 477.5  & 1944.5 ± 1669.0 & 3432.90 ± 1017.62 \\
         DDPG(min-Q)  & 3541.4 ± 726.0    & 4042.5 ± 1697.5  & 2844.9 ± 885.4  & 10279.1 ± 271.3  & 5258.8 ± 422.9 & 5193.34 ± 800.62  \\
        \bfseries MPG &\bfseries 5247.5 ± 539.1  & 4826.2 ± 461.4 & 3143.5 ± 718.4 & 10754.5 ± 418.5 &  5301.3 ± 79.5 & 5854.60 ± 443.38 \\
        \midrule
       TD3    & 4315.9 ± 1499.1   & 4703.3 ± 456.6 & 3350.4 ± 391.2  & 10401.4 ± 1668.8  & 5524.8 ± 1007.1 & 5659.16 ± 1004.56 \\
        \bfseries MPG-SD  & 4687.4 ± 1147.5   &\bfseries 4953.3 ± 578.7 &\bfseries 3406.4 ± 308.4 & 11282.1 ± 354.8 & 5369.6 ± 195.0 & 5939.76 ± 516.88 \\ \midrule
         SAC(min-Q)  & 3908.3 ± 921.8   & 4414.0 ± 409.5 & 2750.5 ± 880.5 & 10337.5 ± 1517.0 & 5743.9 ± 1190.2 & 5430.84 ± 983.80 \\ 
        \bfseries MAC  & 4579.8 ± 833.5   &  3933.1 ± 1023.9 & 2905.8 ± 639.9 & 11339.7 ± 280.4  & 5826.4 ± 779.6 & 5716.96 ± 711.46  \\ 
        \midrule
        TQC  & 4575.5 ± 1960.3   & 4856.3 ± 294.0   & 2397.3 ± 1513.0 & 11361.7 ± 2218.7 &  6959.8 ± 1925.8 &  6030.12 ± 1582.36 \\ 
        \bfseries MQC  & 4861.3 ± 1577.9   & 4418.9 ± 676.0   & 3352.3 ± 596.1 &\bfseries 12029.1 ± 350.0 & \bfseries 7643.2 ± 1229.8 & \bfseries 6460.96 ± 885.96 \\ 

         \bottomrule
        \end{tabular}
        }
        }
\end{table*}

%%%%%%%%%%%%%%%%%%%%%%%%%%%%%%%%%%%%%%%%%%%%%%
\subsection{Moderate Quantile Critics (MQC)}

We describe the proposed MQC, which appropriately incorporates the moderate target in Equation \ref{eq:moderate_target} into TQC \cite{kuznetsov2020controlling}. To train the protester in our MQC, 
the expectile loss function in Equation~\ref{eq:proposed_expectile_value_loss} is modified as
\begin{gather}
    L_{V}(\psi) = \mathbb{E}_{(s, a) \sim \mathcal{D}} \left[ \ell_\tau\left(Q_{\theta}(s, a), V_\psi(s)\right) \right] \nonumber \\
    Q_{\theta}(s,a) = \min \{\kappa_{\theta_n}^m(s, a): m \in [M], n \in [N]\},
    \label{eq:m_tqc_protester_loss}
\end{gather} where $\{\kappa_{\theta_n}^m: m \in [M]\}$ denotes the $M$ atoms of the distribution $Z_{\theta_{n}}(s,a)$ in Equation~\ref{eq:tqc_q_distribution}. As in the proposed MPG and MAC, the protester can be used to reduce overestimation bias in the target distribution in Equation~\ref{eq:tqc_target_distribution}. In MQC, the moderate target distribution is defined as
\begin{gather}
    Y_{\text{mt}}(s, a) = \frac{1}{kN} \sum_{i = 1}^{kN} \delta(y_{\text{mt},i}(s, a)),
    \label{eq:m_tqc_target_distribution} 
\end{gather} where the moderate atoms are given as
\begin{gather}
     y_{\text{mt},i}(s, a) = r + \gamma \Big[ (1 - \omega) z_{(i)}(s', a') +  \omega V_{\psi}(s')  \nonumber \\ 
    \quad\quad\quad\quad - \alpha \log \pi_{\phi}(a' \mid s') \Big], \;\; a' \sim \pi_{\phi}(\cdot|s').
\end{gather}
The loss function for the actor network in MQC is identical to that in TQC, as delineated in Equation~\ref{eq:tqc_policy_loss}.

%%%%%%%%%%%%%%%%%%%%%%%%%%%%%%%%%%%%%%
\subsection{Discussions} 
%%%%%%%%%%%%%%%%%%%%%%%%%%%%%%%%%%%%%%

We discuss some advantages of our moderate target.

\textbf{Simplicity and Flexibility} Our moderate target offers an intuitive solution to the issue of overestimation by simply modifying the update mechanism of the Q-function. Moreover, the proposed algorithms—MPG, MPG-SD, MAC, and MQC—exhibit complexities comparable to their respective underlying algorithms. For instance, MPG-SD is one of the simplest actor-critic algorithms in continuous control environments while yielding an attractive performance by mitigating overestimation bias. Moreover, the conservativeness of our algorithms can be easily adjusted to align with characteristics of the environment by simply tuning the cautious parameter $\omega$.

\textbf{Robustness} Our moderate target employs the lower expectile of the Q-value distribution conditioned on any given state, thus integrating the associated risk information effectively. This risk information is often overlooked in conventional RL algorithms; however, its consideration is vital for the development of robust and stable RL algorithms. As an example, consider a state $s \in \Sc$ having both a high Q-value action $a_1\in\Ac$ and a low Q-value action $a_2\in \Ac$. A typical agent is trained to select the action $a_1$ under the premise that the environment does not contain any uncertainty. In practice, however, unseen events in the environment might lead to the unintended action $a_2$ with some probability (e.g., drone controls) \cite{wang2021online}. By leveraging the moderate target, such potential risk can be integrated into our Q-value estimates, thereby improving the robustness of our algorithms in environments with uncertainty.

\textbf{Expandability} In addition to DDPG, SAC, TD3, and TQC, our moderate target in Equation~\ref{eq:moderate_target} can be readily incorporated into other RL algorithms, such as Q-learning \cite{watkins1992q} and DQN \cite{van2016deep} for discrete control tasks, as well as A2C \cite{mnih2016asynchronous} and PPO \cite{schulman2017proximal} for continuous control tasks. Furthermore, we show that the expectile at an appropriate level $\tau \in (0,1)$ can effectively estimate the various aspects of the Q-value distribution. By simply modifying the expectile level, for example, it is feasible to estimate the upper expectile as well. This capability can further expand the MF-RL algorithms, enabling them to adapt to various environments and significantly enhance overall operational efficiency.

%%%%%%%%%%%%%%%%%%%%%%%%%%%%%%%%%%%%%%%%%%%%
\section{Experiments}\label{sec:experiments}
%%%%%%%%%%%%%%%%%%%%%%%%%%%%%%%%%%%%%%%%%%%%

As benchmark algorithms, we employ SOTA actor-critic methods, including 
DDPG, DDPG(min-Q), SAC(min-Q), TD3, and TQC, owing to their attractive performance and stability in continuous control tasks. Figure~\ref{fig:algorithms} depicts the relationships between the benchmark algorithms and our algorithms. We employ Stable Baselines3 \cite{stable-baselines} framework for our experiments, adhering to the standard hyperparameter settings in RL Zoo3 \cite{rl-zoo3}. Regarding our algorithms, we establish the expectile level at $\tau = 0.01$ to ensure the protester's focus towards the lower expectile. Following the analysis in Section~\ref{sec:ablation_studies}, the cautious weights are assigned as follows: $\omega = 0.2$ for MPG and MPG-SD, $\omega = 0.13$ for MAC, and $\omega = 0.01$ for MQC. 
% {\RED The source code for the implementation of our proposed algorithms is available online.\footnote{\url{https://github.com/ukjohwang/moderate-rl}} }

%%%%%%%%%%%%%%%%%%%%%%%%%%
\subsection{Evaluation}
%%%%%%%%%%%%%%%%%%%%%%%%%%%

We conduct various experiments within the MuJoCo environment \cite{todorov2012physics}, which is widely recognized as a standard benchmark for continuous control tasks. Provided by OpenAI Gym \cite{brockman2016openai}, MuJoCo offers several challenging tasks that rigorously evaluate the performances of both the proposed and benchmark algorithms. Each task is run for 1 million or 2 million time steps with evaluations every 25,000 time steps, where each evaluation reports the average reward over the 5 episodes. Our results are reported over 5 random seeds of the Gym simulator and the network initialization. 
The corresponding results are depicted in Figure~\ref{fig:simulation_graph}. Also, the average reward and the standard deviation over the 5 episodes, evaluated at the last time step, are summarized in Table~\ref{table:simulation_performance_result}.
% {
% \RED
% The corresponding results are summarized in Table~\ref{table:simulation_performance_result}.
% }
Our experiments demonstrate that the proposed algorithms outperform the corresponding counterparts. Specifically, MPG, MPG-SD, MAC, and MQC show superior performance compared to DDPG, TD3, SAC(min-Q), and TQC, respectively. These results confirm that our moderate target effectively addresses the overestimation problem. Our algorithms exhibit lower variances (i.e., greater stability) than their counterparts, and improve both performance and variance without increasing computational complexity, as detailed in Section~\ref{sec:ablation_studies}.

%%%%%%%%%%%%%%%%%%%%%%%%%%%%%%%%%%%%%%%%%%%%%%%%%%%%%%%%%%
\subsection{Ablation Studies}\label{sec:ablation_studies}
%%%%%%%%%%%%%%%%%%%%%%%%%%%%%%%%%%%%%%%%%%%%%%%%%%%%%%%%%%

We conduct ablation studies to evaluate the advantages of our moderate target.

{\bf Efficiency} One may raise concern that our algorithms introduce additional complexity compared to baseline counterparts due to the use of the additional value function (i.e., the protester). However, in our approach, we mitigate this complexity by reducing the number of critics. Specifically, DDPG(min-Q), TD3, and SAC(min-Q) use two critics, whereas MPG, MPG-SD, and MAC use only one critic. Moreover, although MQC employs two critics, identical TQC, we demonstrate in the supplementary material that MQC outperforms TQC despite the increased number of critics. Consequently, our algorithms maintain a similar resource usage as the baseline algorithms. To verify this, we evaluate the resource requirements of both the proposed and benchmark methods, with a particular emphasis on training time and maximum GPU memory usage (in short, Max GPU Mem) in the Ant-v4 environment. The corresponding results are summarized in Table~\ref{table:resource_result}, thereby confirming the efficiency of our algorithms.

%%%%%%%%%%%%%%%%%%%%%%%%%%%%%%%%%%%%%
{\bf Overestimation Bias} To evaluate the overestimation bias during training, we measure the average of the estimated target critic values for various algorithms in the Ant-v4 environment. The corresponding results are depicted in Figure~\ref{fig:target_q_value}. Our findings reveal that the target critic values in DDPG exhibit significant overestimation, whereas the values in the other algorithms remain relatively stable. This discrepancy arises as DDPG does not exploit any technique to control overestimation bias. Notably, although the estimated target critic values in DDPG are significantly higher, this does not result in improved performance, as demonstrated in Table \ref{table:simulation_performance_result}. This disparity highlights the substantial impact of overestimation bias in MF-RL algorithms and emphasize the necessity of addressing this issue to enhance performance.
%%%%%%%%%%%%%%%%%%%%%%%%%%%%%%%%%%
\begin{table}[h]
    \centering
    \renewcommand{\arraystretch}{1.2}
     \caption{Training time and GPU memory usage.} \label{table:resource_result}
        {\small
        \begin{tabular}{ c c c} 
        \toprule
        \bfseries  & \bfseries Time (s) & \bfseries Max GPU Mem (MB) \\ \midrule
        DDPG      & 8364    & 24.0   \\
        DDPG(min-Q)       & 11576    & 26.6   \\
        \bfseries MPG      & 12132    & 26.0   \\ 
       TD3       & 8097    & 26.6  \\ 
        \bfseries MPG-SD       & 8589    & 26.0   \\
       SAC(min-Q)      & 20267    & 23.1   \\
        \bfseries MAC       & 20266    & 22.3  \\ 
        TQC       & 20964    & 35.4  \\ 
        \bfseries  MQC       &   24944  & 36.5  \\ 
         \bottomrule
        \end{tabular}
        }
\end{table}
%%%%%%%%%%%%%%%%%%%%%%%%%%%%%%%%%%%%%%%%%%%%%%%%%%%%%%%%%%%%%%%%%%%%%%%
\begin{figure}[t]
    \centering  
    \includegraphics[width=0.90\linewidth]{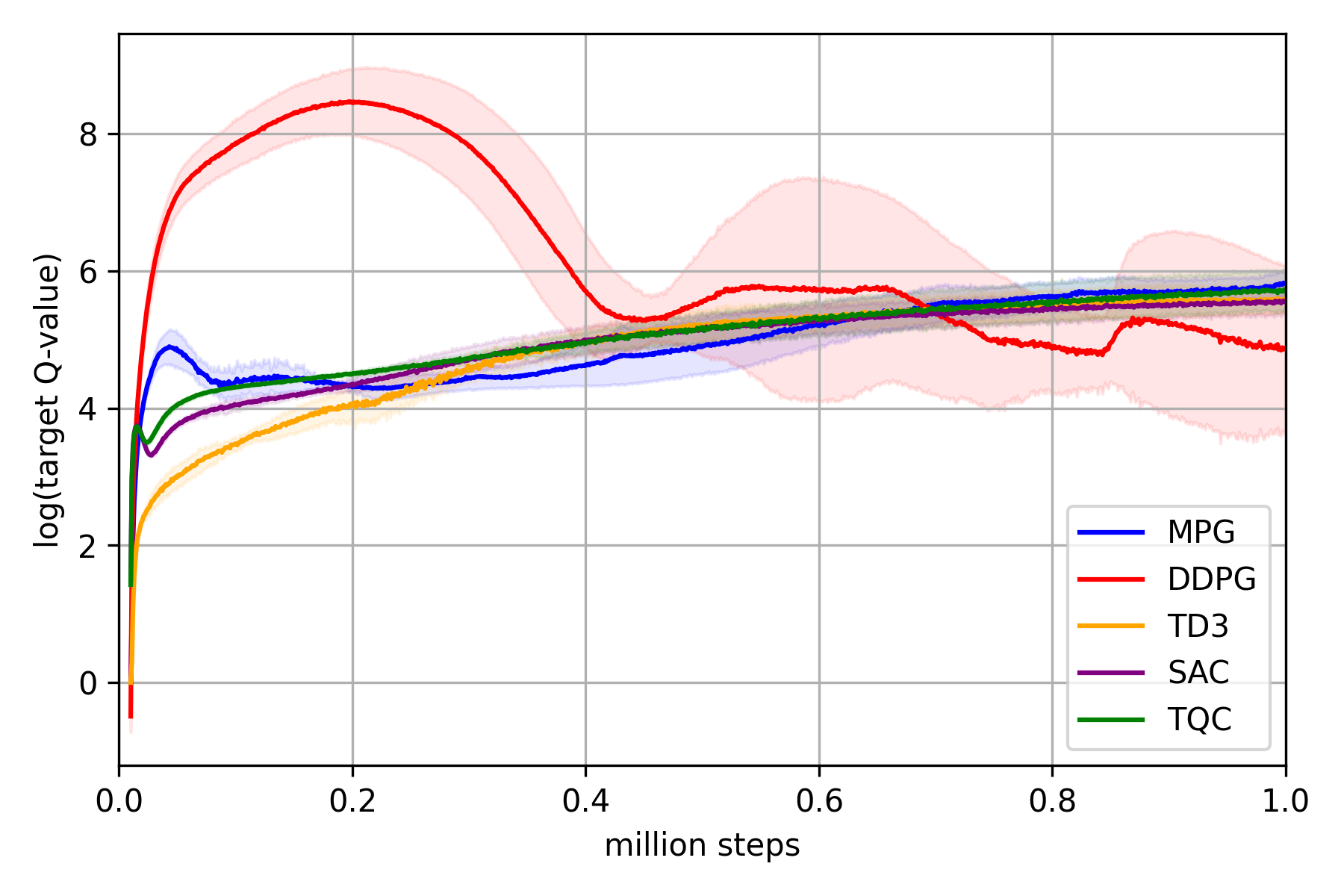} 
    \caption{Log-transformed target Q-value during training.}
    \label{fig:target_q_value}
\end{figure}
%%%%%%%%%%%%%%%%%%%%%%%%%%%%%%%%%%%%%%%%%%%%%%%%%%%%%%%%%%%%%%%%%%%%%%

{\bf Impact of the cautious weight} It is anticipated that an optimal cautious weight $\omega$ would be larger when a baseline algorithm exhibits a higher overestimation bias. In contrast to DDPG and TD3, which utilize deterministic policies, SAC employs a stochastic policy that can mitigate overestimation bias to some extent. TQC offers further mitigation of overestimation through various techniques. Thus, in our baseline methods, DDPG, TD3, SAC, and TQC cause lower overestimation bias in that order, which is also demonstrated in Table 1. From our experiments, the optimized cautious weights are determined to be $\omega=0.2$ for MPG and MPG-SD, $\omega=0.13$ for MAC, and $\omega=0.01$ for MQC. These findings confirm the aforementioned expectations and provide the guidelines on how to select a proper cautious weight.

%Proper selection of the cautious weight $\omega$ is crucial. If $\omega$ is too high, the algorithm becomes overly pessimistic, while if it is too low, it becomes too optimistic. Compared to DDPG and TD3, which use deterministic policies, SAC adopts a maximum entropy framework, resulting in a policy with higher entropy and thus being more conservative with respect to overestimation bias. As a result, the cautious weight in MAC is slightly lower than in MPG and MPG-SD. Additionally, MQC incorporates various techniques to reduce overestimation bias, such as dropping lower quantiles and using two critics. This enables MQC to handle overestimation bias to some extent, requiring a much lower cautious weight compared to the other algorithms.

%%%%%%%%%%%%%%%%%%%%%%%%%%%%%%%%%%%%%%%%%%%
\section{Conclusion}\label{sec:conclusion}

We proposed a novel moderate target designed for controlling overestimation bias. This target can be seamlessly integrated with the SOTA MF-RL algorithms and effectively combined with the existing techniques to further mitigate overestimation. Via extensive experiments, we demonstrated that our algorithms, built upon the moderate target, outperform the corresponding baseline algorithms while maintaining similar computational complexity and training time. These results highlight the potential of the moderate target as a key component to build MF-RL algorithms applicable to various large-scale environments.

%{\BLUE We developed a novel model-free RL algorithm for continuous control tasks, termed MPG, specifically designed to address the overestimation bias prevalent in existing methods. By introducing protester, MPG effectively mitigates this bias, providing a straightforward and intuitive approach compared to the more complex techniques found in other RL algorithms. Through extensive experiments, we demonstrated that MPG achieves competitive performance and efficiency when compared to established RL algorithms. These results underscore MPG's potential as a baseline RL algorithm across various applications.
%}

\section*{Impact Statement}

This paper presents work whose goal is to advance the field of Machine Learning. There are many potential societal consequences of our work, none of which we feel must be specifically highlighted here.

\nocite{langley00}

\bibliography{example_paper}
\bibliographystyle{icml2025}

%%%%%%%%%%%%%%%%%%%%%%%%%%%%%%%%%%%%%%%%%%%%%%%%%%%%%%%%%%%%%%%%%%%%%%%%%%%%%%%
%%%%%%%%%%%%%%%%%%%%%%%%%%%%%%%%%%%%%%%%%%%%%%%%%%%%%%%%%%%%%%%%%%%%%%%%%%%%%%%
% APPENDIX
%%%%%%%%%%%%%%%%%%%%%%%%%%%%%%%%%%%%%%%%%%%%%%%%%%%%%%%%%%%%%%%%%%%%%%%%%%%%%%%
%%%%%%%%%%%%%%%%%%%%%%%%%%%%%%%%%%%%%%%%%%%%%%%%%%%%%%%%%%%%%%%%%%%%%%%%%%%%%%%
\newpage
\appendix
\onecolumn

\section{Convergence of the Moderate Bellman Equation}

In this section, we theoretically prove the convergence of the proposed moderate Bellman equation, which is defined by the incorporation of an expectile value function into the Bellman operator.  The formal definition of the moderate Bellman equation is represented as:
\begin{equation}
    \mathcal{T}_m Q(s, a) \coloneqq r(s, a) + \gamma \sum_{s' \in \mathcal{S}} p(s' \mid s, a) \left[ (1 - \omega) \max_{a' \in \mathcal{A}} Q(s', a') + \omega V(s') \right],
    \label{eq:moderate_bellman_equation}
\end{equation}
where $Q$ is the action-value function, $V$ represents the expectile value function, and $ \omega \in [0, 1] $ is a cautious weight. To analyze convergence, we make the following assumption:

\textbf{Assumption 1.}  
The expectile value function $ V(s) $ approximates the minimum of the action-value function over all possible actions at  a given state $ s $:
\begin{equation}
    V(s) = \min_{a \in \mathcal{A}} Q(s, a), \quad \forall s \in \mathcal{S}.
\end{equation}

Under Assumption 1, we can rewrite the moderate Bellman equation as follows:
\begin{equation}
    \mathcal{T}_m Q(s, a) \coloneqq r(s, a) + \gamma \sum_{s' \in \mathcal{S}} p(s' \mid s, a) \left[ (1 - \omega) \max_{a' \in \mathcal{A}} Q(s', a') + \omega \min_{a' \in \mathcal{A}} Q(s', a') \right].
\end{equation}

We next prove the convergence of the moderate Bellman operator $ \mathcal{T}_m $ using the contraction mapping theorem.

\begin{theorem}
For any $ \gamma \in (0, 1) $ and $ \omega \in [0, 1] $, the moderate Bellman operator $ \mathcal{T}_m $ in Equation~\ref{eq:moderate_bellman_equation} is a contraction with respect to the $ l_\infty $-norm. Consequently, the action-value function $ Q $ has a unique fixed point.
\end{theorem}

\begin{proof}
For two arbitrary functions $ Q_1 $ and $ Q_2 $, we analyze the difference $ |\mathcal{T}_m Q_1(s, a) - \mathcal{T}_m Q_2(s, a)| $:
\begin{align}
    &|\mathcal{T}_m Q_1(s, a) - \mathcal{T}_m Q_2(s, a)| \nonumber \\
    &= \left | r(s, a) + \gamma \sum_{s^\prime \in \mathcal{S}} p(s'\mid s, a) \left[ (1 - \omega)\max_{a'\in \mathcal{A}} Q_1(s', a') + \omega V_1(s')  \right] \right. \nonumber \\
    &\left. - r(s, a) + \gamma \sum_{s^\prime \in \mathcal{S}} p(s'\mid s, a) \left[ (1 - \omega)\max_{a'\in \mathcal{A}} Q_2(s', a') + \omega V_2(s')  \right]  \right | \nonumber \\
    &= \left |  \gamma (1 - \omega) \sum_{s^\prime \in \mathcal{S}} p(s'\mid s, a) \left[ \max_{a'\in \mathcal{A}} Q_1(s', a') - \max_{a'\in \mathcal{A}} Q_1(s', a')   \right] +  \gamma \cdot \omega \sum_{s^\prime \in \mathcal{S}} p(s'\mid s, a) \left[ V_1(s') - V_2(s')  \right]  \right | \nonumber \\
    &\stackrel{(a)}{=} \left |  \gamma (1 - \omega) \sum_{s^\prime \in \mathcal{S}} p(s'\mid s, a) \left[ \max_{a'\in \mathcal{A}} Q_1(s', a') - \max_{a'\in \mathcal{A}} Q_1(s', a')   \right] +  \gamma \cdot \omega \sum_{s^\prime \in \mathcal{S}} p(s'\mid s, a) \left[ \min_{a'\in \mathcal{A}} Q_1(s', a') - \min_{a'\in \mathcal{A}} Q_1(s', a')  \right]  \right | \nonumber \\
    &\le \gamma (1 - \omega) \sum_{s^\prime \in \mathcal{S}} p(s'\mid s, a) \left |   \max_{a'\in \mathcal{A}} Q_1(s', a') - \max_{a'\in \mathcal{A}} Q_1(s', a')  \right |  +  \gamma \cdot \omega \sum_{s^\prime \in \mathcal{S}} p(s'\mid s, a) \left | \min_{a'\in \mathcal{A}} Q_1(s', a') - \min_{a'\in \mathcal{A}} Q_1(s', a') \right | \nonumber \\
    & \le \gamma(1-\omega) \left \| Q_1 - Q_2 \right \|_\infty + \gamma \cdot \omega \left \| Q_1 - Q_2 \right \|_\infty \nonumber \\
    &= \gamma \left \| Q_1 - Q_2 \right \|_\infty
\end{align}
where (a) is due to the {\bf Assumption 1.} Since $ \gamma \in (0, 1) $, the operator $ \mathcal{T}_m $ is a contraction mapping with respect to the $ l_\infty $-norm. According to the Banach fixed-point theorem \cite{puterman2014markov}, $ Q $ has a unique fixed point under $ \mathcal{T}_m $. This completes the proof.

\end{proof}

\section{Algorithm Details}

In this section, we delineate the comprehensive procedures for the proposed algorithms: MPG, MPG-SD, MAC, and MQC. The MPG and MPG-SD algorithms are presented in Algorithm~\ref{alg:mpg_algorithm}, the MAC algorithm is outlined in Algorithm~\ref{alg:mac_algorithm}, and the MQC algorithm is described in Algorithm~\ref{alg:mqc_algorithm}.

\begin{algorithm}[h]
    \caption{The Proposed MPG-SD Algorithm} \label{alg:mpg_algorithm}
    \begin{algorithmic}[1] 
        \STATE {\bf Input:} Discount factor $\gamma \in (0, 1)$, learning rates $\lambda_\pi, \lambda_Q, \lambda_V > 0$, cautious weight $\omega \in [0, 1]$, expectile level $\tau \in (0, 0.5)$, target update rate  $\eta \in (0, 1)$, random noise $\epsilon$ for exploration, clipped random noise $\bar{\epsilon}$ for target policy smoothing, and delay parameter $d$.

        \STATE {\bf Initialization:} Network parameters  $\phi, \theta, \psi$, target-network parameters $\bar{\phi} \leftarrow \phi, \bar{\theta} \leftarrow \theta$, replay buffer $\mathcal{D} = \emptyset$, initial state $s_0 \in \mathcal{S}$, policy update time $u = 0$.

        \FOR{each epoch}
            % \FOR{each environment step}
            \FOR{$t=0$ to $T-1$}
                % \STATE $a_t = \pi_{\phi}(s_t) + {\RED n_t},\; n_t \sim \mathcal{N}(0,\sigma^2)$
                \STATE  $a_t = \pi_{\phi}(s_t) +  \epsilon$
                \STATE ${s_{t+1}} \sim p(\cdot \mid s_t, a_t)$
                \STATE $\mathcal{D} \leftarrow \mathcal{D} \cup \{ (s_{t} ,a_t, r(s_{t},a_{t}), s_{t+1}) \}$
                % \STATE $s \leftarrow s'$
                \STATE $s_t \leftarrow s_{t+1}$
            
                \FOR{each update step}
                    \STATE $u \leftarrow u + 1$
                    \STATE Sampling $(s, a, r, s') \sim \mathcal{D}$
                    \STATE Update $\theta$, and $\psi$ via \eqref{eq:proposed_critic_loss}, and \eqref{eq:proposed_expectile_value_loss}, respectively
    
                    \IF{$u \bmod d = 0$}
                        \STATE Update $\phi$ via \eqref{eq:ddpg_policy_update}
                        \STATE Update $\bar{\phi}$ and $\bar{\theta}$ via \eqref{eq:target_update}
                    \ENDIF
                \ENDFOR
            \ENDFOR
        \ENDFOR
    \end{algorithmic}
    $\Diamond$ If $d = 1$ and $\bar{\epsilon} = 0$, MPG-SD is simplified as MPG.
\end{algorithm}

\begin{algorithm}[h!]
    \caption{The Proposed MAC Algorithm} \label{alg:mac_algorithm}
    \begin{algorithmic}[1] 
        \STATE {\bf Input:} Discount factor $\gamma \in (0, 1)$, learning rates $\lambda_\pi, \lambda_Q, \lambda_V > 0$, cautious weight $\omega \in [0, 1]$, expectile level $\tau \in (0, 0.5)$, target update rate  $\eta \in (0, 1)$, and initial temperature parameter $\alpha$.
        
        \STATE {\bf Initialization:} Network parameters  $\phi, \theta, \psi$, target-network parameters $\bar{\theta} \leftarrow \theta$, replay buffer $\mathcal{D} = \emptyset$, and initial state $s_0 \in \mathcal{S}$.
                
        \FOR{each epoch}
            \FOR{$t=0$ to $T-1$}
                \STATE  $a_t \sim \pi_{\phi}(\cdot \mid s_t)$
                \STATE ${s_{t+1}} \sim p(\cdot \mid s_t, a_t)$
                \STATE $\mathcal{D} \leftarrow \mathcal{D} \cup \{ (s_{t} ,a_t, r(s_{t},a_{t}), s_{t+1}) \}$
                \STATE $s_t \leftarrow s_{t+1}$
            
                \FOR{each update step}
                    \STATE Sampling $(s, a, r, s') \sim \mathcal{D}$
                    \STATE Update $\phi$, $\theta$, and $\psi$ via \eqref{eq:mac_policy_update} \eqref{eq:proposed_critic_loss}, and \eqref{eq:proposed_expectile_value_loss}, respectively
                    \STATE Update $\bar{\theta}$ via soft update mechanism
                
                \ENDFOR
            \ENDFOR
        \ENDFOR
    \end{algorithmic}
\end{algorithm}

\begin{algorithm}[h]
    \caption{The Proposed MQC Algorithm} \label{alg:mqc_algorithm}
    \begin{algorithmic}[1] 
        \STATE {\bf Input:} Discount factor $\gamma \in (0, 1)$, learning rates $\lambda_\pi, \lambda_Q, \lambda_V > 0$, cautious weight $\omega \in [0, 1]$, expectile level $\tau \in (0, 0.5)$, target update rate  $\eta \in (0, 1)$, initial temperature parameter $\alpha$, number of critic networks $N$, number of atoms per networks $M$, and number of selected atoms per network $k$.
        
        \STATE {\bf Initialization:} Network parameters  $\phi, \theta, \psi$, target-network parameters $\bar{\theta} \leftarrow \theta$, replay buffer $\mathcal{D} = \emptyset$, and initial state $s_0 \in \mathcal{S}$.
                
        \FOR{each epoch}
            \FOR{$t=0$ to $T-1$}
                \STATE  $a_t \sim \pi_{\phi}(\cdot \mid s_t)$
                \STATE ${s_{t+1}} \sim p(\cdot \mid s_t, a_t)$
                \STATE $\mathcal{D} \leftarrow \mathcal{D} \cup \{ (s_{t} ,a_t, r(s_{t},a_{t}), s_{t+1}) \}$
                \STATE $s_t \leftarrow s_{t+1}$
            
                \FOR{each update step}
                    \STATE Sampling $(s, a, r, s') \sim \mathcal{D}$
                    \STATE Update $\phi$, and $\psi$ via \eqref{eq:tqc_policy_loss} \eqref{eq:m_tqc_protester_loss}, respectively
                    \STATE Update $\theta$ via Huber quantile loss using \eqref{eq:m_tqc_target_distribution}
                    \STATE Update $\bar{\theta}$ via soft update mechanism
                
                \ENDFOR
            \ENDFOR
        \ENDFOR
    \end{algorithmic}
\end{algorithm}

\section{Additional Experiments}

In this section, we present supplementary experimental results alongside the hyperparameters utilized in our algorithms.

\subsection{Number of critics in TQC}
While MPG, MPG-SD, and MAC clearly do not require additional resources compared to their baseline algorithms, MQC appears to have increased complexity relative to TQC. This is because MQC uses two critics and one protester, whereas TQC only uses two critics. To demonstrate the efficiency of MQC compared to TQC, we conducted experiments with an increased number of critics in TQC. Specifically, we evaluated TQC with 1, 2, 3, and 5 critics, denoted as TQC(N=1), TQC(N=2), TQC(N=3), and TQC(N=5), respectively. The results of these experiments, summarized in Table~\ref{table:tqc_simulation}, show that MQC outperforms both TQC(N=3) and even TQC(N=5). These findings highlight the superiority of incorporating the moderate target into TQC.

\begin{table*}[t]
    \centering
     \caption{
    Average reward and standard deviation calculated after training over five episodes for TQC based algorithms, with training conducted across five different seeds.
     }
    \label{table:tqc_simulation}
    \renewcommand{\arraystretch}{1.2}
        \resizebox{\textwidth}{!}{
        \begin{tabular}{ c c c c c c c c} 
        \toprule
            & \bfseries Ant-v4 & \bfseries Walker2d-v4 & \bfseries Hopper-v4  & \bfseries HalfCheetah-v4 & \bfseries Humanoid-v4 & \bfseries Average \\ \midrule

        TQC(N=1) &  2214.4 ± 1917.5 & 3909.9 ± 1695.2   &  2040.9 ± 1431.1  &11721.8 ± 450.6  & 7696.9 ± 1097.0 & 5516.78 ± 1582.36     \\ 
        TQC(N=2)  & 4575.5 ± 1960.3   & 4856.3 ± 294.0   & 2397.3 ± 1513.0 & 11361.7 ± 2218.7 &  6959.8 ± 1925.8 &  6030.12 ± 1582.36 \\ 
        TQC(N=3)  & 4286.1 ± 1735.5    & 4076.8 ± 1684.9    & 2885.6 ± 1022.9  & 12389.1 ± 215.6  & 7254.6 ± 2570.6   & 6178.44 ± 1445.90   \\ 
        TQC(N=5)  &  3645.9 ± 1842.6  &  4711.2 ± 665.6  & 3462.1 ± 125.4 & 12036.4 ± 440.1 & 8102.5 ± 314.7  &  6391.62 ± 677.68 \\ \midrule
        \bfseries MQC(N=2)  & 4861.3 ± 1577.9   & 4418.9 ± 676.0   & 3352.3 ± 596.1 & 12029.1 ± 350.0 &  7643.2 ± 1229.8 &  6460.96 ± 885.96 \\ 

         \bottomrule
        \end{tabular}
        }
\end{table*}

\section{Hyperparameters.}
In this section, we provide detailed hyperparameter settings for the proposed algorithms and the baseline benchmark algorithms mentioned in the main text.

\begin{table*}[h!]
    \centering
    
    \resizebox{\textwidth}{!}{
    \begin{tabular}{c c}
        
        \begin{tabular}{ c c } 
            \toprule
             Hyperparameter & Common \\ \midrule
            Batch Size &  256 \\
            Buffer size &  $1 \cdot 10^6$ \\
            Optimizer  &  Adam \\ 
            Discount factor ($\gamma$) &  0.99 \\
            Non-Linearity &  ReLU \\
            Target update rate ($\eta$) & 0.005 \\
            Total time steps & $1 \cdot 10^6$ ($2 \cdot 10^6$ for Humanoid-v4) \\
            \bottomrule
        \end{tabular}
        &
        
        \begin{tabular}{c c c c c} 
            \toprule
             Hyperparameter & SAC & MAC & TQC & MQC \\ \midrule
            Hidden layer dimension  & \multicolumn{4}{c}{[256, 256]} \\
            Learning rate ($\lambda_{\pi}, \lambda_{Q}, \lambda_{V}$) & \multicolumn{4}{c}{3e-4} \\
            Number of critics ($N$) & 2 & 1 & 2 & 2 \\
            Number of protesters & - & 1 & - & 1 \\
            Number of atoms ($M$) & - & - & 25 & 25 \\
            Number of selected atoms per network(k) & - & - & 23 & 23 \\
            Expectile level ($\tau$)  & - & 0.01 & - & 0.01   \\
            Cautious weight ($\omega$)  & - & 0.13 & - & 0.01 \\
            \bottomrule
        \end{tabular}
    \end{tabular}
    }

    \vspace{0.5cm} 
    \resizebox{\textwidth}{!}{
    \begin{tabular}{ c  c c c c c} 
        \toprule
         Hyperparameter & DDPG & DDPG(min-Q) &  MPG & TD3 & MPG-SD \\ \midrule
        Hidden layer dimension  & \multicolumn{5}{c}{[400, 300]} \\
        Learning rate ($\lambda_{\pi}, \lambda_{Q}, \lambda_{V}$) & \multicolumn{5}{c}{1e-3 (3e-4 for Humanoid-v4)} \\
        Action noise ($\epsilon$)  & \multicolumn{5}{c}{$\mathcal{N}(0, 0.1^2)$} \\
        Number of critics ($N$) & 1 & 2 & 1 & 2 & 1 \\
        Number of protesters & - & - & 1 & - & 1 \\
        Policy and target update interval(d) & 1 & 1 & 1 & 2 & 2   \\
        Target action noise ($\bar{\epsilon}$)  & - & - & - & $\text{clip}(\mathcal{N}(0, 0.2), -0.5, 0.5)$ & $\text{clip}(\mathcal{N}(0, 0.2), -0.5, 0.5)$ \\
        Expectile level ($\tau$)  & - & - & 0.01 & - & 0.01 \\
        Cautious weight ($\omega$)  & - & - & 0.2 & - & 0.2 \\
        \bottomrule
    \end{tabular}
    }

    \caption{Hyperparameter settings.}
\end{table*}


% You can have as much text here as you want. The main body must be at most $8$ pages long.
% For the final version, one more page can be added.
% If you want, you can use an appendix like this one.  

% The $\mathtt{\backslash onecolumn}$ command above can be kept in place if you prefer a one-column appendix, or can be removed if you prefer a two-column appendix.  Apart from this possible change, the style (font size, spacing, margins, page numbering, etc.) should be kept the same as the main body.
%%%%%%%%%%%%%%%%%%%%%%%%%%%%%%%%%%%%%%%%%%%%%%%%%%%%%%%%%%%%%%%%%%%%%%%%%%%%%%%
%%%%%%%%%%%%%%%%%%%%%%%%%%%%%%%%%%%%%%%%%%%%%%%%%%%%%%%%%%%%%%%%%%%%%%%%%%%%%%%

\end{document}